\documentclass{article}
\usepackage{amsmath, amssymb, amsthm}
\usepackage{enumerate}
\usepackage{spconf,graphicx}
\usepackage[ruled,linesnumbered]{algorithm2e}
\usepackage{booktabs,multirow}
\usepackage{subcaption}
\usepackage{hyperref}

\newcommand{\bfphi}{\boldsymbol{\phi}}
\newcommand{\bftheta}{\boldsymbol{\theta}}
\newcommand{\bfalpha}{\boldsymbol{\alpha}}
\newcommand{\bfmu}{\boldsymbol{\mu}}
\newcommand{\loss}{\mathcal{L}}
\newcommand{\reg}{\mathcal{R}}
\newcommand{\blk}{\mathcal{B}}
\newcommand{\tasks}{\mathcal{T}}
\newcommand{\lin}{\text{lin}}
\newcommand{\diag}{\text{diag}}
\newcommand{\data}[1]{\mathcal{D}^{\mathrm{#1}}}
\newcommand{\argmin}{\operatornamewithlimits{argmin}}
\newcommand{\argmax}{\operatornamewithlimits{argmax}}

\newtheorem{theorem}{Theorem}
\SetKwInput{KwInit}{Initialization}{}{}

\title{Meta-Learning with Versatile Loss Geometries\\ for Fast Adaptation Using Mirror Descent}
\name{Yilang Zhang, Bingcong Li, Georgios B. Giannakis \thanks{This work was supported by NSF grants 2126052, 2128593, 2212318, 2220292, and 2312547; Emails: \{zhan7453,lixx5599,georgios\}@umn.edu}}
\address{Dept. of ECE, University of Minnesota, Minneapolis, MN 55455, USA}
\begin{document}
%
\setlength{\abovedisplayskip}{4pt}
\setlength{\belowdisplayskip}{4pt}
\maketitle
\begin{abstract}
Utilizing task-invariant prior knowledge extracted from related tasks, meta-learning is a principled framework that empowers learning a new task especially when data records are limited. 
A fundamental challenge in meta-learning is how to quickly ``adapt'' the extracted prior in order to train a task-specific model within \textit{a few} optimization steps. Existing approaches deal with this challenge using a preconditioner that enhances convergence of the per-task training process. Though effective in representing locally a quadratic  training loss, these simple linear preconditioners can hardly capture complex loss geometries. The present contribution addresses this limitation by learning a nonlinear mirror map, which induces a versatile distance metric to enable capturing and optimizing a wide range of loss geometries, hence facilitating the per-task training. Numerical tests on few-shot learning datasets demonstrate the superior expressiveness and convergence of the advocated approach. 
\end{abstract}
\begin{keywords}
Meta-learning, bilevel optimization, mirror descent, loss geometries
\end{keywords}
\section{Introduction}
\label{sec:intro}
The success of deep learning relies heavily on large-scale and high-dimensional models, which require extensive training using a large number of data. However, this ``data-driven learning'' approach is not feasible in applications where data are scarce due to costly data collection and labelling process. Examples of such applications include drug discovery~\cite{drug}, machine translation~\cite{machine-trans}, and robot manipulation~\cite{robot}. 

In contrast, \emph{meta-learning} offers a powerful approach for learning a task in data-limited setups. Specifically, meta-learning extracts \textit{task-invariant prior} information from a collection of given tasks, that can subsequently aid learning of a new, albeit related task. Although this new task may have limited training data, the prior serves as a strong inductive bias that effectively transfers knowledge to aid its learning. In image classification for instance, a feature extractor learned from a collection of given tasks can act as a common prior, and thus benefit a variety of other image classification tasks. 

Depending on how this ``data-limited learning'' is performed, meta-learning algorithms can be categorized into neural network (NN)- and optimization-based ones. In NN-based ones, the per-task learning is viewed as an NN mapping from its training data to task-specific model parameters~\cite{Meta-LSTM, neural-attentive}. The prior information is encoded in the NN weights, which are shared and optimized across tasks. With the universality of NNs in approximating complex mappings granted, their black-box structure challenges their reliability and interpretability. On the other hand, optimization-based meta-learning alternatives interpret ``data-limited learning'' as a cascade of a few optimization iterations (a.k.a. adaptation) over the model parameters. The prior here is captured by the shared hyperparameters of the iterative optimizer. A representative of these alternatives is the model-agnostic meta-learning (MAML)~\cite{MAML}, which views the prior as a learnable task-invariant initialization of the optimizer. By starting from an informative initial point, the model parameters can rapidly converge to local minima within a few gradient descent (GD) steps. 
Building upon MAML, a series of variants have been proposed to learn different priors~\cite{iMAML,MetaOptNet,iBaML}. 

While optimization-based meta-learning has been proven effective numerically, recent studies suggest that its generalization and stability heavily rely on convergence of per-task optimization~\cite{iMAML, iBaML}. This motivates one to grow the number of descent iterations. However, this can be infeasible as the overall complexity of meta-learning scales linearly with the number of GD steps~\cite{iMAML}. Besides, using accelerated first-order optimizers, such as Adam~\cite{Adam}, introduces extra backpropagation complexity when optimizing the prior. To improve the per-task convergence without markedly adding to the complexity, another line of research focuses on second-order optimization using a learnable precondition matrix having simple form~\cite{MetaSGD, MetaMD,MetaCurvature, MT-net,WarpGrad, ECML}. In fact, the precondition matrix captures the local quadratic curvature of the training loss, and linearly transforms the gradient based on this curvature. To acquire more expressive preconditioners, recent advances suggest replacing the linear matrix multiplication with a nonlinear NN transformation~\cite{MetaKFO}. However, convergence of this NN-manipulated GD is an uncharted territory. 

The present work advocates learning a generic distance metric induced by a strictly increasing nonlinear mirror map, which enables efficient optimization over generic loss geometries. All in all, our contribution is three-fold.
\begin{enumerate}[i)]
	\vspace{-.15cm}
	\setlength{\itemsep}{-.075cm}
	\item Broadening linear preconditioners with guaranteed per-task convergence. 
	\item Blockwise inverse autoregressive flow (blockIAF) ensuring monotonicity and scalability of the mirror map. 
	\item Numerical tests showing superior performance and improved convergence compared to linear preconditioners. 
	\vspace{-.49cm}
\end{enumerate}

\section{Problem setup}
To enable ``data-limited learning'' of a new task, meta-learning forms task-invariant priors using a collection of given tasks indexed by $t = 1,\ldots,T$. Each task comprises a dataset $\data{}_t := \{ (\mathbf{x}_t^n, y_t^n) \}_{n=1}^{N_t}$ consisting of $N_t$ data-label pairs, which are split into a training subset $\data{trn}_t$, and a disjoint validation subset $\data{val}_t$. The new task, indexed by $\star$, contains a training subset $\data{trn}_\star$, and a set of test data $\{ \mathbf{x}_\star^n \}_{n=1}^{N_\star^{\mathrm{tst}}}$ for which the corresponding labels $\{ y_\star^n \}_{n=1}^{N_\star^{\mathrm{tst}}}$ are to be predicted. The key premise of meta-learning is that all the aforementioned tasks share related model structures or data distributions. Thus, one can postulate a large model shared across all tasks, along with distinct model parameters $\bfphi_t \in \mathbb{R}^d$ per individual task. But since the cardinality $N_t^\mathrm{trn} := |\data{trn}_t|$ can be much smaller than $d$, learning a task by directly optimizing $\bfphi_t$ over $\data{trn}_t$ is impractical. Fortunately, since $T$ is considerably large, a task-invariant prior can be learned using $\{ \data{val}_t \}_{t=1}^T$ to render per-task learning well posed. 

Letting $\bftheta \in \mathbb{R}^{d'}$ denote the vector parameter of the prior, the meta-learning objective can be formulated as a bilevel optimization problem. The lower-level trains each task-specific model by optimizing $\bfphi_t$ using $\data{trn}_t$ and $\bftheta$ from the upper-level. The upper-level adjusts $\bftheta$ by evaluating the optimized $\bfphi_t$ on the validation sets $\{ \data{val}_t \}_{t=1}^T$. The two levels depend on each other and yield the following nested objective
\begin{subequations}
\label{eq:bilevel-global}
\begin{align}
	\label{eq:bilevel-global-upper}
	&\min_{\bftheta} \sum_{t=1}^T \loss (\bfphi_t^* (\bftheta); \data{val}_t) \\
	\label{eq:bilevel-global-lower}
	&~\text{s.t.} ~~\bfphi_t^* (\bftheta) = \argmin_{\bfphi_t} \loss (\bfphi_t; \data{trn}_t) + \reg (\bfphi_t; \bftheta), ~\forall t
\end{align}
\end{subequations}
where $\loss$ is the loss function capturing each task-specific model fit, and $\reg$ is the regularizer accounting for the task-invariant prior. From the Bayesian viewpoint, $\mathcal{L}$ and $\reg$ represent the negative log-likelihood (nll), $-\log p(\mathbf{y}_t^{\mathrm{trn}} | \bfphi_t; \mathbf{X}_t^{\mathrm{trn}})$, and the negative log-prior (nlp) $-\log p(\bfphi_t ; \bftheta)$, where $\mathbf{X}_t^{\mathrm{trn}} := [\mathbf{x}_t^1,\ldots,\mathbf{x}_t^{N_t^\mathrm{trn}}]$ and $\mathbf{y}_t^\mathrm{trn} := [y_t^1,\ldots,y_t^{N_t^\mathrm{trn}}]^\top$ ($^\top$ denotes transpose). Bayes' rule then implies $\bfphi_t^* = \argmin - \log p(\bfphi_t | $ $\mathbf{y}_t^{\mathrm{trn}}$; $\mathbf{X}_t^{\mathrm{trn}}, \bftheta)$ is the maximum a posteriori (MAP) estimator. 

Reaching the global optimum $\bfphi_t^*$ is generally infeasible because the task-specific model is nonlinear. Hence, a prudent remedy is to rely on an approximate solver $\hat{\bfphi}_t \approx \bfphi_t^*$ obtained by a tractable optimizer. For instance, MAML replaces~\eqref{eq:bilevel-global-lower} with a $K$-step GD minimizing the nll:
\begin{equation}
\label{eq:GD}
	\bfphi_t^{(k)} (\bftheta) = \bfphi_t^{(k-1)} (\bftheta) - \alpha \nabla \loss (\bfphi_t^{(k-1)} (\bftheta); \data{trn}_t),~\forall t
\end{equation}
where $k=1,\ldots,K$ indexes iterations; initialization $\bfphi_t^{(0)} = \bfphi^{(0)} = \bftheta$; approximate solver $\hat{\bfphi}_t (\bftheta) = \bfphi_t^{(K)} (\bftheta)$; and $\alpha$ denotes the step size. Although $\reg (\bfphi_t; \bftheta) = 0$ in MAML, it has been shown that the GD solver satisfies~\cite{LLAMA}
\begin{equation*}
	\hat{\bfphi}_t (\bftheta) \approx \bfphi_t^* (\bftheta) = \argmin_{\bfphi_t} \loss (\bfphi_t; \data{trn}_t) + \frac{1}{2} \| \bfphi_t - \bftheta \|_{\mathbf{\Lambda}_t}^2, ~\forall t
\end{equation*}
where the precision matrix $\mathbf{\Lambda}_t$ is determined by $\nabla^2 \mathcal{L} (\bftheta; \data{trn}_t)$, $\alpha$, and $K$. This indicates that MAML's optimization strategy~\eqref{eq:GD} is approximately tantamount to an implicit Gaussian prior probability density function (pdf) $p(\bfphi_t; \bftheta) = \mathcal{N} (\bftheta, \mathbf{\Lambda}_t^{-1})$, with the task-invariant initialization serving as the mean vector. Alongside implicit priors, their explicit counterparts have also been investigated with various prior pdfs~\cite{iMAML, iBaML}. 

For both implicit and explicit priors, numerical studies~\cite{MetaSGD, MetaCurvature} and theoretical analyses~\cite{iMAML, iBaML} demonstrate that the gradient error for optimizing $\bftheta$ in~\eqref{eq:bilevel-global-upper} relies on the convergence accuracy of $\hat{\bfphi}_t$ relative to a stationary point. In addition, employing a large $K$ or complicated optimizers could prohibitively escalate the overall complexity for solving~\eqref{eq:bilevel-global}. As a consequence, attention has been directed towards preconditioned GD (PGD) solvers, as in the update
\begin{equation}
\label{eq:precond-GD}
	\bfphi_t^{(k)} (\bftheta) = \bfphi_t^{(k-1)} (\bftheta) - \alpha \mathbf{P} (\bftheta_P) \nabla \loss (\bfphi_t^{(k-1)} (\bftheta); \data{trn}_t)
\end{equation}
where $\bftheta_P$ parametrizes $\mathbf{P} \in \mathbb{R}^{d\times d}$, and the prior parameter is augmented as $\bftheta := [\bfphi^{(0)\top}, \bftheta_P^\top]^\top$. To ensure~\eqref{eq:precond-GD} incurs affordable complexity after preconditioning, $\mathbf{P}$ must have a simple enough structure so that $\mathbf{P} (\bftheta_P) \nabla \loss (\bfphi_t^{(k-1)}; \data{trn}_t)$ incurs computational complexity $\mathcal{O}(d)$. Examples of such structures include diagonal~\cite{MetaSGD, MetaMD}, block-diagonal~\cite{MetaCurvature, MT-net}, and NN-based~\cite{WarpGrad} matrices. A more generic preconditioner can be formed by replacing the linear transformation $\mathbf{P} (\bftheta_P) \nabla \loss (\bfphi_t^{(k-1)}; \data{trn}_t)$ with a nonlinear NN $f(\nabla \loss (\bfphi_t^{(k-1)}; \data{trn}_t); \bftheta_P)$~\cite{MetaKFO}, but unfortunately convergence of this alternative iterate may not be guaranteed. 

Essentially, GD conducts a pre-step greedy search with a quadratic loss approximation. To see this, let $\lin(\loss (\bfphi_t),\bar{\bfphi}_t)$ $:= \loss(\bar{\bfphi}_t; \data{trn}_t) + (\bfphi_t - \bar{\bfphi}_t)^\top \nabla \loss(\bar{\bfphi}_t; \data{trn}_t)$. Using this linearization of $\loss$ at $\bar{\bfphi}_t \in \mathbb{R}^d$, the GD update reduces to (cf.~\eqref{eq:GD})
\begin{equation}
\label{eq:GD-greedy}
	\hspace{-0.05cm}\bfphi_t^{(k)} = \argmin_{\bfphi_t} \lin(\loss(\bfphi_t), \bfphi_t^{(k-1)}) + \frac{1}{2\alpha} \| \bfphi_t - \bfphi_t^{(k-1)} \|_2^2
\end{equation}
where dependencies on $\bftheta$ are dropped hereafter for notational brevity. The term $\frac{1}{2\alpha} \| \bfphi_t - \bfphi_t^{(k-1)} \|_2^2$ implies the isotropic approximation $\nabla^2 \loss (\bfphi_t^{(k-1)}; \data{trn}_t) \approx \frac{1}{\alpha} \mathbf{I}_d$, while~\eqref{eq:precond-GD} refines the approximation as a more informative matrix $\frac{1}{\alpha}\mathbf{P}^{-1}$ (if invertible). This quadratic local approximation is particularly effective when $K$ is large and $\alpha$ is small, which gradually ameliorates $\bfphi_t^{(k)}$ to a stationary point. In meta-learning however, the standard setup relies on a small $K$ (e.g., $1$ or $5$) and a sufficiently large $\alpha$, so that the model can quickly adapt to the task with low complexity. This tradeoff highlights the need for learning more expressive loss geometries. 

\section{Loss Geometries using Mirror Descent}
Instead of quadratic approximations of the local loss induced by certain norms (e.g., $\| \cdot \|_2$ and $\| \cdot \|_{\mathbf{P}^{-1}}$), our fresh idea is a data-driven distance metric that captures a broader spectrum of loss geometries. This is accomplished by learning the so-termed ``mirror map,'' which will be introduced first. All the proofs are delegated to Appendix~\ref{app:proof}. 

\subsection{Modeling the loss geometry using the mirror map}
To generalize the (P)GD, we will replace the $\ell_2$-norm in~\eqref{eq:GD-greedy} with a generic metric $D_h$ to arrive at
\begin{equation}
\label{eq:MD-greedy}
	\bfphi_t^{(k)} = \argmin_{\bfphi_t} \lin(\loss(\bfphi_t), \bfphi_t^{(k-1)}) + \frac{1}{\alpha} D_h (\bfphi_t, \bfphi_t^{(k-1)})
\end{equation}
where $D_h (\bfphi_t, \bfphi_t^{(k-1)}) := h(\bfphi_t) - \lin (h(\bfphi_t), \bfphi_t^{(k-1)})$ is the Bregman divergence, and the associated distance-generating function $h: \mathbb{R}^d \mapsto \mathbb{R}$ is strongly convex to ensure the existence and uniqueness of the minimizer. As a result, $\nabla h$ is strictly increasing, and thus invertible\footnote{When $\nabla h$ is discontinuous but $h$ is proper, the inverse $(\nabla h)^{-1}$ is defined as $\nabla h^* (\mathbf{z}) := \argmax_{\bfphi} \bfphi^\top \mathbf{z} - h(\bfphi)$, where $h^* (\mathbf{z}) := \sup_{\bfphi} \bfphi^\top \mathbf{z} - h(\bfphi)$ is the Fenchel conjugate of $h$.}. Then, applying the optimality condition leads to the mirror descent (MD) update
\begin{equation}
\label{eq:MD}
	\bfphi_t^{(k)} = (\nabla h)^{-1} \big( \nabla h(\bfphi_t^{(k-1)}) - \alpha \nabla \loss (\bfphi_t^{(k-1)}; \data{trn}_t) \big).
\end{equation}
The invertible $\nabla h$, dubbed mirror map, connects $\bfphi_t$ in the primal space to $\nabla \loss$ in the dual space under the endowed metric $D_h$. As a special case, when choosing $h(\cdot) = \frac{1}{2}\| \cdot \|_2^2$, it is easy to verify that~\eqref{eq:MD} boils down to~\eqref{eq:GD} due to the self-duality of the $\ell_2$-norm. Likewise,~\eqref{eq:precond-GD} can be obtained with $h(\cdot) = \frac{1}{2}\| \cdot \|_{\mathbf{P}^{-1}}^2$, where $\nabla h$ reduces to a linear mapping. Function $h$ reflects our prior knowledge about the geometry of $\loss$. In particular, letting $h(\cdot) = \loss (\cdot; \data{trn}_t)$ (even when $\loss$ is not strong convex) in~\eqref{eq:MD-greedy} gives $\bfphi_t^{(k)} = \argmin_{\bfphi_t} \loss (\bfphi_t; \data{trn}_t)$, which is precisely the original nll minimization solved in~\eqref{eq:GD} and~\eqref{eq:precond-GD}. Thus, an ideal choice of $h$ would yield $h \approx \loss$ (up to a constant) within a sufficiently large region around $\bfphi_t^{(k-1)}$. 

Different from past works that rely on a simple preselected $h$ to model loss geometries, we here acquire a data-driven $h$ by learning a strictly increasing $\nabla h$ that best fits the given tasks. Interestingly,~\eqref{eq:MD} can be reformulated to yield an update of the dual vector $\mathbf{z}_t := \nabla h(\bfphi_t)$ as
\begin{equation}
\label{eq:MD-dual}
	\mathbf{z}_t^{(k)} = \mathbf{z}_t^{(k-1)} - \alpha \nabla \loss \big( (\nabla h)^{-1}(\mathbf{z}_t^{(k-1)}); \data{trn}_t \big)
\end{equation}
with $\mathbf{z}_t^{(0)} = \nabla h(\bfphi^{(0)})$ and $\hat{\bfphi}_t = (\nabla h)^{-1} (\mathbf{z}_t^{(K)})$. Hence, it suffices to learn a strictly increasing $(\nabla h)^{-1}$ and a task-invariant dual initialization $\mathbf{z}^{(0)} := \nabla h(\bfphi^{(0)})$, thus removing the need for directly calculating $\nabla h$. 

\subsection{Learning the inverse mirror map via blockIAF}
Inspired by this observation, a prudent option is to model $(\nabla h)^{-1}$ as an inverse autoregressive flow (IAF)~\cite{IAF}. The notable benefit of IAF lies in its efficient parallelization of forward computation, that makes it considerably faster than computing its inverse. However, directly applying the dimension-wise IAF to the high-dimensional $\mathbf{z}_t \in \mathbb{R}^d$ will incur prohibitively high complexity of $\Omega(d^2)$. For this reason, we introduce a novel blockIAF model that effectively reduces complexity by performing block-wise (nonlinear) autoregression on a low-dimensional space encoding $\mathbf{z}_t$. To this end, let $\{ \blk_i \}_{i=1}^B$ be a partition of the index set $\{ 1,\ldots,d \}$, and $[\mathbf{z}_t]_{\blk_i}$ denote the subvector of $\mathbf{z}_t$ restricted to the block $\blk_i$. The blockIAF model transforms $\mathbf{z}_t$ to $\bfphi_t$ through
\begin{subequations}
\label{eq:blkIAF}
\begin{align}
	\label{eq:blkIAF-scale-shift}
	[\bfphi_t]_{\blk_i} &= [\mathbf{z}_t]_{\blk_i} \odot \sigma ( \bfalpha_i ) + \bfmu_i \\
	\label{eq:blkIAF-encode-decode}
	[\bfalpha_i^\top, \bfmu_i^\top]^\top &= d_i \big( \{ e_j([\mathbf{z}_t]_{\blk_j}) \}_{j=1}^{i-1} \big), ~i=1,\ldots,B
\end{align}	
\end{subequations}
where nonlinearity $\sigma$ is positive and upper bounded (e.g., logistic function), $\sigma(\bfalpha_i), \bfmu_i \in \mathbb{R}^{|\blk_i|}$ are the scale and shift of $[\mathbf{z}_i]_{\blk_i}$, $e_i$ and $d_i$ denote learnable encoder and decoder for the $i$-th block, and $\odot$ is the Hadamard (element-wise) product. In our implementation, $\{ e_i \}_{i=1}^{B-1}$ and $\{ d_i \}_{i=1}^B$ are multilayer perceptrons (MLPs) with ReLU activations. To further reduce complexity, all linear layers in MLPs are implemented by tensor mode product~\cite{MetaCurvature}. This technique is equivalent to a low-rank Kronecker approximation to MLPs' weight matrices. This lowers the per-step MD complexity to $\mathcal{O} (d)$. 

The following theorem characterizes two important properties of the proposed blockIAF model. 

\begin{theorem}
\label{theo:blkIAF}
Let $g: \mathbb{R}^{d} \mapsto \mathbb{R}^{d}$ be the blockIAF model~\eqref{eq:blkIAF}. For any partition $\{ \blk_i \}_{i=1}^B$, $g$ is strictly increasing, that is
\begin{equation}
\label{eq:strict-inc}
	(\mathbf{z}_t - \mathbf{z}_t')^\top (g(\mathbf{z}_t) - g(\mathbf{z}_t')) > 0, ~~\forall \mathbf{z}_t \ne \mathbf{z}_t'.
\end{equation}
Moreover, there exists a constant $C > 0$ such that 
\begin{equation}
    \nabla (g^{-1}) (\bfphi_t) \succeq C.
\end{equation}
\end{theorem}

Theorem~\ref{theo:blkIAF} asserts that with $(\nabla h)^{-1} = g$, one ensures the desired strict monotonicity, and strong convexity of the induced $h$ (by noting that $\nabla^2 h = \nabla (g^{-1})$). As a result, the per-task optimization~\eqref{eq:MD-dual} enjoys the standard convergence guarantee of MD. Although the convergence rate of MD is in the same order as GD, it outperforms GD markedly in the constant factor when $d$ is large~\cite{MD-highdim}, and relies on more relaxed assumptions~\cite{SMD}. 

\begin{table*}[t]
\caption{Comparison of meta-learning algorithms with different loss geometry models on the $5$-class miniImageNet dataset. Maximum and mean accuracies within its $95\%$ confidence interval are in bold. (No 
model ensembling for a fair comparison.)}
\vspace{-0.7cm}
\label{tab:perf-comp}
\begin{center}
\begin{tabular}{@{}lcccc@{}}
\toprule
\multirow{2}*{Method} & \multirow{2}*{Lower-level optimizer} & \multirow{2}*{Loss geometry model} & \multicolumn{2}{c}{$5$-class accuracies}  \\
 & & & $1$-shot & $5$-shot \\
\midrule
MAML~\cite{MAML} & GD & identity matrix & $48.70 \pm 1.84 \%$ & $63.11 \pm 0.92 \%$ \\
MetaSGD~\cite{MetaSGD} & PGD & diag. matrix & $50.47 \pm 1.87 \%$ & $64.03 \pm 0.94 \%$ \\
MT-net~\cite{MT-net} & PGD & block diag. matrix & $51.70 \pm 1.84 \%$ & $-$ \\
WarpGrad~\cite{WarpGrad} & PGD & NN-based low-rank matrix & $52.3 \pm 0.8 \%$ & $68.4 \pm 0.6 \%$ \\
MetaCurvature
\cite{MetaCurvature} & PGD & block diag. \& Kron. (low-rank) matrix & $54.23 \pm 0.88 \%$ & $67.99 \pm 0.73 \%$ \\
MetaKFO~\cite{MetaKFO} & NN-transformed GD & NN-based gradient transformation & $-$ & $64.9 \%$ \\
ECML~\cite{ECML} & PGD & Gauss-Newton approximation & $48.94 \pm 0.80 \%$ & $65.26 \pm 0.67 \%$ \\
\hline
This paper's method & MD & blockIAF-based mirror map & $\mathbf{56.10 \pm 1.43\%}$  & $\mathbf{69.59 \pm 0.71\%}$ \\
\bottomrule
\end{tabular}
\end{center}
\end{table*}

\begin{figure*}[t]
\vspace{-0.77cm}
\centering
\begin{subfigure}[t]{0.385\textwidth}
\includegraphics[width=\columnwidth]{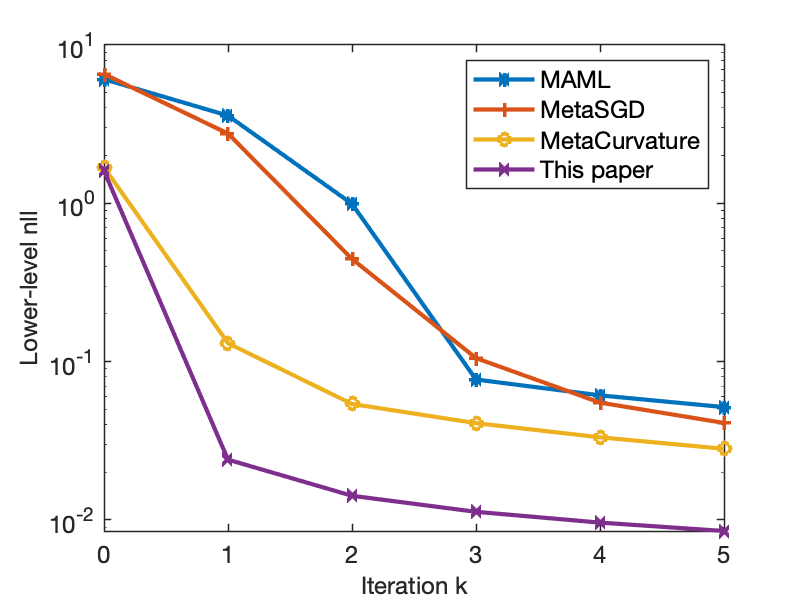}
\vspace{-0.5cm}
\caption{$\loss (\bfphi_\star^{(k)}; \data{trn}_\star)$ versus $k$}
\label{subfig:nll}
\end{subfigure}
\begin{subfigure}[t]{0.385\textwidth}
\includegraphics[width=\columnwidth]{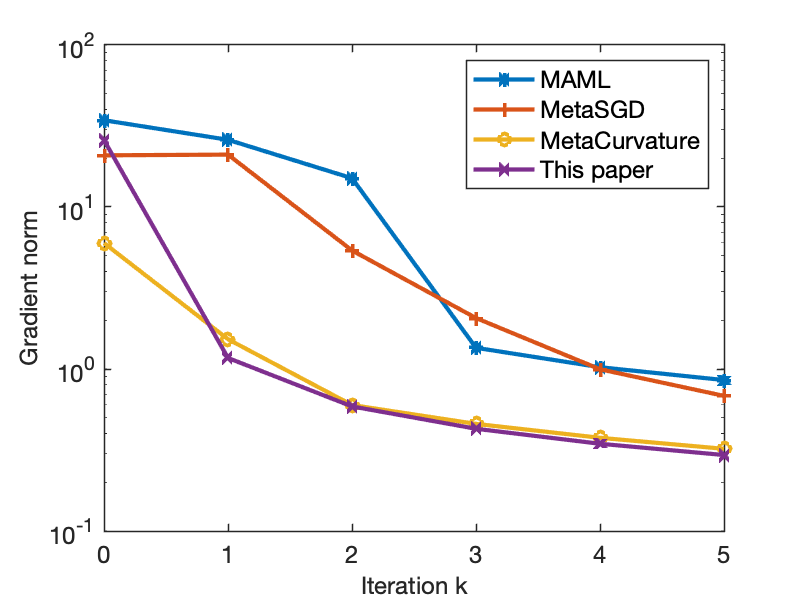}
\vspace{-0.5cm}
\caption{$\| \nabla \loss (\bfphi_\star^{(k)}; \data{trn}_\star) \|_2$ versus $k$}
\label{subfig:grad-norm}
\end{subfigure}
\vspace{-0.25cm}
\caption{Convergence comparison on randomly sampled new tasks.}
\label{fig:convergence}
\vspace{-0.45cm}
\end{figure*}

The meta-learning objective~\eqref{eq:bilevel-global} is solved using alternating optimization. With $\bftheta_g$ denoting the blockIAF parameters, let $\bftheta := [\mathbf{z}^{(0)\top}, \bftheta_g^\top]^\top$ be the prior parameter vector. In the $(r)$-th iteration of~\eqref{eq:bilevel-global-upper}, the optimizer has access to $\bftheta^{(r-1)}$ provided by its last iteration, and a batch of randomly sampled tasks $\tasks^{(r)} \subset \{1,\ldots,T\}$. The optimizer first solves $\hat{\bfphi}_t (\bftheta^{(r-1)})$ for each $t \in \tasks^{(r)}$ leveraging the $K$-step MD~\eqref{eq:MD-dual}. Then, $\bftheta^{(r-1)}$ is updated using mini-batch stochastic GD with step size $\beta$:
\begin{equation*}
	\bftheta^{(r)} = \bftheta^{(r-1)} - \beta\frac{T}{|\tasks^{(r)}|} \sum_{t \in \tasks^{(r)}} \nabla_{\bftheta^{(r-1)}} \loss (\hat{\bfphi}_t (\bftheta^{(r-1)}); \data{val}_t).
\end{equation*}
A summary of the algorithm can be found in Appendix~\ref{app:alg}. 

\section{Numerical tests}
Here we compare the empirical performance of optimization-based meta-learning using different lower-level optimizers, on the standard few-shot classification dataset miniImageNet~\cite{MatchingNets}, where ``shots'' signify the per-class training data for each $t$. The task-specific model is a standard $4$-layer convolutional NN (CNN)~\cite{MatchingNets, MAML}. Each layer comprises a $3 \times 3$ convolution of $64$ channels, batch normalization, ReLU activation, and $2 \times 2$ max pooling module. After the convolutional layers, a linear regressor with softmax activation is appended to perform classification. Subset $\blk_i$ is formed by the weight indices of the $i$-th CNN layer. The autoregression in~\eqref{eq:blkIAF-encode-decode} implies that ``how to optimize weights of the $i$-th layer'' depends on ``how weights of previous layers have been optimized.'' This choice enables blockIAF to model the optimization dependency of high-level features (e.g., textures and patterns) on low-level ones (e.g., colors and edges). Test setups and hyperparameters can be found in Appendix~\ref{app:test-setup}. 

Table~\ref{tab:perf-comp} lists various loss geometry models, where classification accuracy on new tasks is the figure of merit. For fairness, MAML is the backbone of all methods. By utilizing a more versatile loss geometry model, our approach outperforms the state-of-the-art ones by a large margin. 

To further gauge the performance gain achieved by our novel approach, Fig.~\ref{fig:convergence} visualizes the convergence of $\loss (\bfphi_\star^{(k)}; \data{trn}_\star)$ averaged on $1,000$ random new tasks. The proposed method results in faster convergence to a lower and more stable nll compared with all three competitors. Moreover, Fig.~\ref{subfig:nll} reveals that both the proposed method and MetaCurvature improve the initialization compared to MAML and MetaSGD. This confirms that convergence and generalization of~\eqref{eq:bilevel-global-upper} relies on the convergence accuracy of $\hat{\bfphi}_t$~\cite{iMAML,iBaML}. Fig.~\ref{subfig:grad-norm} further illustrates that although the initial gradients of different methods have comparable norms $\| \nabla \loss (\bfphi_\star^{(0)}; \data{trn}_\star) \|_2$, our method can make better use of the gradient, leading to a rapid reduction of the nll as well as its gradient norm at $k=1$. This improved gradient utilization highlights our method's superior modeling of loss geometries.

\section{Conclusions and outlook}
Versatile loss geometry models can accelerate the lower-level convergence in meta-learning. A novel BlockIAF model is introduced to learn the inverse mirror map $(\nabla h)^{-1}$ induced by a strongly convex $h$. The resultant algorithm generalizes preconditioning-based meta-learning, captures versatile loss geometries, and improves lower-level convergence. Effectiveness of the novel approach was validated on a standard few-shot dataset. Future research includes bi-level convergence guarantees for the proposed method, and development of more expressive yet scalable inverse mirror maps. 
\clearpage

\bibliographystyle{IEEEbib}
\bibliography{refs}


\clearpage
\onecolumn
\appendix
\setcounter{theorem}{0}
\begin{center}
{\huge \bf Appendix}
\end{center}

\section{Proof of Theorem 1}
\label{app:proof}
\begin{theorem}[Restated]
Let $g: \mathbb{R}^{d} \mapsto \mathbb{R}^{d}$ denote the blockIAF model~\eqref{eq:blkIAF}. For any partition $\{ \blk_i \}_{i=1}^B$, $g$ is strictly increasing, that is 
\begin{equation*}
	(\mathbf{z}_t - \mathbf{z}_t')^\top (g(\mathbf{z}_t) - g(\mathbf{z}_t')) > 0, ~~\forall \mathbf{z}_t \ne \mathbf{z}_t'.
\end{equation*}
Moreover, there exists a constant $C > 0$ such that 
\begin{equation*}
    \nabla (g^{-1}) (\bfphi_t) \succeq C.
\end{equation*}
\end{theorem}

\begin{proof}
Let $\pi := [\blk_1, \ldots, \blk_B]$ denote a permutation of $\{1,\ldots,n\}$, and $\mathbf{Q}_{\pi} \in \mathbb{R}^{d \times d} := \big[ [\mathbf{I}_d]_{\blk_1}, \ldots, [\mathbf{I}_d]_{\blk_B} \big]$ the permutation matrix under $\pi$, where $[ \mathbf{I}_d ]_{\blk_i}$ is the submatrix of the identity $\mathbf{I}_d \in \mathbb{R}^{d \times d}$ restricted to the columns indexed by $\blk_i$. 

Consider the partial derivatives (cf.~\eqref{eq:blkIAF})
\begin{equation}
\label{eq:partial-deriv}
	\frac{\partial [ g(\mathbf{z}_t)]_{\blk_i}}{\partial [\mathbf{z}_t]_{\blk_j} } = \frac{\partial [\bfphi_t (\mathbf{z}_t)]_{\blk_i}}{\partial [\mathbf{z}_t]_{\blk_j} } = 
	\begin{cases}
		\text{a $|\blk_i| \times |\blk_j|$ matrix}, &\text{if}~i > j \\
		\diag(\sigma(\bfalpha_i)), &\text{if}~i=j \\
		\mathbf{0}_d, &\text{otherwise}
	\end{cases}.
\end{equation}
It can be verified that the Jacobian $\nabla_{[\mathbf{z}_t]_\pi} [g(\mathbf{z}_t)]_\pi$ of the permuted parameters is block-upper-triangular, with the $i$-th diagonal block given by $\frac{\partial [g(\mathbf{z}_t)]_{\blk_i}}{\partial [\mathbf{z}_t]_{\blk_i} } = \diag(\sigma(\bfalpha_i)) \succ 0$. It thus holds that $\nabla_{[\mathbf{z}_t]_\pi} [g(\mathbf{z}_t)]_\pi \succ 0$, or equivalently, $\mathbf{Q}_\pi^\top \nabla g(\mathbf{z}_t) \mathbf{Q}_\pi \succ 0$, which implies that 
\begin{equation}
\label{eq:g-deriv-PD}
	\nabla g(\mathbf{z}_t) \succ 0,~\forall \mathbf{z}_t \in \mathbb{R}^d \;.
\end{equation} 
Letting $\tilde{g} (\alpha) := g(\alpha \mathbf{z}_t + (1-\alpha) \mathbf{z}_t')$, it holds for $\forall \mathbf{z}_t \ne \mathbf{z}_t'$
\begin{align}
	(\mathbf{z}_t - \mathbf{z}_t')^\top (g(\mathbf{z}_t) - g(\mathbf{z}_t')) &= (\mathbf{z}_t - \mathbf{z}_t')^\top (\tilde{g}(1) - \tilde{g}(0)) \nonumber \\
    &= (\mathbf{z}_t - \mathbf{z}_t')^\top \int_0^1 \tilde{g}'(\alpha)d\alpha  \nonumber \\
	&= \int_0^1 (\mathbf{z}_t - \mathbf{z}_t')^\top \nabla g(\alpha \mathbf{z}_t + (1-\alpha) \mathbf{z}_t') (\mathbf{z}_t - \mathbf{z}_t') d\alpha > 0
\end{align}
where the inequality follows from~\eqref{eq:g-deriv-PD}. 

Next, upper bounding $\sigma \le 1/C$, we will show that $\nabla (g^{-1}) (\bfphi_t) \succeq C$ for some constant $C > 0$. To obtain the inverse $g^{-1}$, notice that~\eqref{eq:blkIAF-scale-shift} can be readily rewritten as
\begin{equation*}
		[\mathbf{z}_t]_{\blk_i} = ([\bfphi_t]_{\blk_i} - \bfmu_i) \odot 1 / \sigma (\bfalpha_i ). 
\end{equation*}
where $/$ is the element-wise division. Similar to~\eqref{eq:partial-deriv}, it can be easily verified that the Jacobian $\nabla_{[\bfphi_t]_\pi} [(g^{-1})(\bfphi_t)]_\pi$ is also block-upper-triangular, with $i$-th diagonal block $\frac{\partial [(g^{-1})(\bfphi_t)]_{\blk_i}}{\partial [\bfphi_t]_{\blk_i} } = \diag^{-1}(\sigma(\bfalpha_i)) \succeq C$. 

As a result, we have that 
\begin{equation*}
	\nabla (g^{-1})(\bfphi_t) \succeq \min_{i=1,\ldots,B} 1 / \| \sigma ( \bfalpha_i ) \|_{\infty} \succeq C
\end{equation*}
which completes the proof.
\end{proof}

\section{Summary of the algorithm}
\label{app:alg}

\begin{algorithm}[h]
\caption{Meta-learning with MD and blockIAF}
\KwIn{$\{ \data{}_t \}_{t=1}^T$, step sizes $\alpha$ and $\beta$, maximum number of iterations $K$ and $R$, and blockIAF mirror map $\nabla h$.}
\KwInit{randomly initialize $\bftheta^{(0)} = [\mathbf{z}^{(0)\top}, \bftheta_g^\top]^\top$.}
\For{$r = 1, \ldots, R$}{
Randomly sample a mini-batch of tasks $\tasks^{(r)} \subset \{ 1,\ldots, T\}$\;
\For{$t \in \mathcal{T}^{(r)}$}{
	Initialize $\mathbf{z}_t^{(0)} = \mathbf{z}^{(0)}$\;
	\For{$k = 1, \ldots, K$}{
		Map $\bfphi_t^{(k-1)} (\bftheta^{(r-1)}) = (\nabla h)^{-1} (\mathbf{z}_t^{(k-1)} (\bftheta^{(r-1)}); \bftheta_g^{(r-1)})$\; 
  		Descend $\mathbf{z}_t^{(k)} (\bftheta^{(r-1)}) = \mathbf{z}_t^{(k-1)} (\bftheta^{(r-1)}) - \alpha \nabla \loss \big( \bfphi_t^{(k-1)} (\bftheta^{(r-1)}); \data{trn}_t \big)$\;
        }
    Map $\hat{\bfphi}_t (\bftheta^{(r-1)}) = (\nabla h)^{-1} (\mathbf{z}_t^{(K)} (\bftheta^{(r-1)}); \bftheta_g^{(r-1)})$\;
    }
Update $\bftheta^{(r)} = \bftheta^{(r-1)} - \beta \frac{T}{|\mathcal{T}^r|} \sum_{t \in \mathcal{T}^r} \nabla_{\bftheta^{(r-1)}} \loss(\hat{\bfphi}_t (\bftheta^{(r-1)}); \data{val}_t)$\;
}
\KwOut{$\hat{\bftheta} = \bftheta^{(R)}$.}
\end{algorithm}

\section{Numerical setups}
\label{app:test-setup}
This section elaborates further on the dataset and setups of the numerical tests.  

The miniImageNet dataset is a few-shot classification dataset comprising natural images from $100$ classes, each containing $600$ samples. All images are cropped and resized to $84 \times 84$, as suggested by~\cite{Meta-LSTM}. The $100$ classes are disjointly divided into $3$ groups with corresponding size $64$, $20$ and $16$, which are available to the meta-training, meta-validation, and meta-testing phases, respectively. The task setups follow from the standard $M$-class $N$-shot few-shot learning protocol~\cite{Meta-LSTM, MAML}. In particular, $\data{trn}_t$ per task $t$ contains $M$ classes randomly drawn from the dataset, each consisting of $N$ labeled data. It is easy to see that $| \data{trn}_t | = MN, ~\forall t$. Likewise, $\data{val}_t$ is constructed in a manner akin to $\data{trn}_t$, albeit with each class comprising $15$ labeled data. 

The hyperparameters used in the tests are the same as those used by MAML~\cite{MAML}, and are listed in Table~\ref{tab:hyperparam}. Our implementation relies on PyTorch, and codes are available at~\url{https://github.com/zhangyilang/MetaMirrorDescent}. 

\begin{table*}[h]
\caption{Hyperparameter setup for the numerical tests.}
\vspace{-0.5cm}
\label{tab:hyperparam}
\begin{center}
\begin{tabular}{lcccc}
\toprule
Hyperparameter & Notation & Value  \\
\midrule
Lower-level iterations & $K$ & $5$ \\
Lower-level learning rate & $\alpha$ & $10^{-2}$ \\
Upper-leve iterations & $R$ & $60,000$ \\
Upper-level learning rate & $\beta$ & $10^{-3}$ \\
Upper-level SGD batch size & $|\tasks^{(r)}|$ & $4$ \\
\bottomrule
\end{tabular}
\end{center}
\end{table*}

All the MLPs used in blockIAF have three fully-connected layers with ReLU nonlinearity, and with the weight matrix of each layer Kronecker factorized~\cite{MetaCurvature}. Let $\text{size}_i := d_{i,1} \times d_{i,2} \times \ldots \times d_{i,O_i}$ be the size of the original tensor corresponding to the vector $[\bfphi_t]_{\blk_i}$, where $O_i$ is the total order of the tensor, and $\prod_{j=1}^{O_i} d_{i,j} = |\blk_i|$. Each layer of the encoder $e_i$ outputs a tensor with  dimensionality of half size. This implies that the output tensor of the $l$-th layer of $e_i$ has size $\lfloor \frac{\text{size}_i}{2^l} \rfloor := \lfloor \frac{d_{i,1}}{2^l} \rfloor \times \lfloor \frac{d_{i,2}}{2^l} \rfloor \times \ldots \times \lfloor \frac{d_{i,O_i}}{2^l} \rfloor, ~l=1,2,3$. The decoder $d_i$ first vectorizes and concatenates the embeddings provided by $\{ e_j \}_{j=1}^{i-1}$, maps this concatenated embedding vector to $\lfloor \frac{\text{size}_i}{8} \rfloor$, and recovers the tensor to $\text{size}_i$ by performing the inverse size operations of $e_i$; that is, its $l$-th layer changes the tensor size from $\lfloor \frac{\text{size}_i}{2^{4-l}} \rfloor$ to $\lfloor \frac{\text{size}_i}{2^{3-l}} \rfloor$. 

\section{Complexity analysis}
Next, complexity comparison is implemented to justify the effectiveness of the introduced blockIAF model. To showcase the computational efficiency, numerical complexities are assessed using the 5-class 5-shot miniImageNet dataset. In the test, the blockIAF-based mirror map incurs a $9.1\%$ increase of forward and backpropagation time compared to the basic GD update in MAML. This slight increment confirms the claimed low complexity of the proposed approach.
\end{document}